%% file: main.tex
\documentclass[11pt,letterpaper]{article}

\usepackage[T1]{fontenc}
\usepackage{lmodern}
\usepackage[margin=1in]{geometry}
\usepackage{amsmath,amssymb,amsthm}
\usepackage{microtype}
\usepackage{booktabs}
\usepackage[round,authoryear]{natbib}
\usepackage{xcolor}
\usepackage{adjustbox}
\usepackage{placeins}
\usepackage{hyperref}
\hypersetup{
  colorlinks=true,
  linkcolor=blue,
  citecolor=blue,
  urlcolor=blue,
  pdftitle={Fast Last-Iterate Convergence in Zero-Sum Markov Games with Bandit Feedback},
  pdfauthor={Yuheng Zhang}
}
\newtheorem{theorem}{Theorem}
\newtheorem{lemma}[theorem]{Lemma}
\newtheorem{corollary}[theorem]{Corollary}
\newtheorem{assumption}{Assumption}

\newcounter{algorithm}
\newenvironment{algobox}[1]{%
  \begin{figure}[!t]
  \refstepcounter{algorithm}%
  \begin{adjustbox}{max totalsize={\textwidth}{0.92\textheight},center}
  \begin{minipage}{\textwidth}
  \small
  \setlength{\abovedisplayskip}{3pt}
  \setlength{\belowdisplayskip}{3pt}
  \setlength{\abovedisplayshortskip}{3pt}
  \setlength{\belowdisplayshortskip}{3pt}
  \setlength{\parindent}{0pt}
  \hrule\smallskip
  \textbf{Algorithm \thealgorithm. #1}\par\smallskip
  \hrule\smallskip
}{%
  \par\smallskip\hrule
  \end{minipage}
  \end{adjustbox}
  \end{figure}
}

\title{Fast Last-Iterate Convergence in Zero-Sum Markov Games \\ with Bandit Feedback}

\author{Yuheng Zhang\\
University of Illinois Urbana-Champaign\\
\texttt{yuhengz2@illinois.edu}}
\date{}

\begin{document}

\maketitle

\begin{abstract}
We study last-iterate convergence in unknown two-player zero-sum
discounted Markov games with bandit feedback.
The players learn independently along a single trajectory without
observing each other's actions.
We develop Adaptive Regularized TD Learning (ARTD), which achieves a
$\widetilde{\mathcal O}(t^{-1/4})$ duality gap bound for the current
policies under a uniform hitting time assumption, with high probability
simultaneously over all rounds and starting states.
This improves the $\widetilde{\mathcal O}(t^{-1/(9+\nu)})$ rate of
\citet{cai2023uncoupled}, for any fixed $\nu>0$, under the same feedback
model and hitting time assumption.
Our algorithm requires no knowledge of the hitting time bound, the time
horizon, or the confidence level.
To stabilize policy learning as value estimates change, we separate fast
temporal difference averaging from bounded value updates.
We adapt log-barrier regularization to the progress of value estimation,
controlling both policy and value errors throughout learning.
Together, these mechanisms enable fast convergence of the policies
actually played, even when the players learn independently from bandit
feedback.
\end{abstract}

\input{introduction}
\input{prelim}
\input{algorithm}
\input{theory}
\input{related_work}
\input{conclusion}

\subsection*{AI use statement}
We use GPT-6 Astra to polish the writing, assist with calculations
in the proofs, and check their correctness. We review all AI-assisted
content and take full responsibility for the final content of this paper.

\bibliography{ref}
\bibliographystyle{plainnat}

\clearpage
\appendix
\input{proofs}

\end{document}

%% file: introduction.tex
\section{Introduction}
\label{sec:introduction}

Markov games provide a standard framework for multi-agent reinforcement
learning, in which agents' decisions jointly determine rewards and state
transitions \citep{shapley1953stochastic,littman1994markov}.
They model competitive tasks such as board games and real-time strategy
games, where each player's actions affect both players' immediate payoffs
and the evolution of the game \citep{wei2021markov}.

We study last-iterate convergence in unknown two-player zero-sum
discounted Markov games. The players interact along a single trajectory:
at each round, they choose actions at the current state, receive their
own payoffs, and observe the next state. Both the reward function and
transition dynamics are unknown. Each player learns independently from
its own observations, without observing the opponent's actions or
policies and without communication or shared randomness. This
\emph{uncoupled} learning model asks whether individually implemented
learning rules can lead the two players to a Nash equilibrium
\citep{wei2021markov,sayin2021decentralized,cai2023uncoupled}.

Many sample-efficient methods for learning Markov games provide
convergence guarantees for averages or mixtures of past policies
\citep{bai2020near,jin2024vlearning}.
However, these guarantees do not imply that the policies actually played
converge to equilibrium.
Achieving this stronger guarantee, known as \emph{last-iterate convergence},
is challenging because each player must learn from noisy feedback against
a changing opponent, without averaging away fluctuations in its policies.

Prior work establishes last-iterate guarantees under several feedback
models. With full gradient feedback, optimistic and extragradient
methods converge in matrix games and related saddle point problems
\citep{daskalakis2019last,wei2021linear,cai2022tight}. Under bandit feedback, \citet{fiegel2026optimal} obtain a
$\widetilde{\mathcal O}(t^{-1/4})$ high-probability rate for uncoupled
learning in matrix games using log-barrier regularization.
Achieving the same rate in Markov games is substantially more challenging:
players must learn policies while estimating future rewards from a single
trajectory that visits states unevenly. 

For Markov games with full information about rewards and transitions,
entropy-regularized policy optimization also gives finite-time
last-iterate guarantees \citep{cen2021fast,cen2023faster}.
For unknown Markov games, decentralized methods based on observed rewards
and state transitions provide guarantees for the final policies of
batched or nested learning procedures
\citep{wei2021markov,chen2024decentralized}. Our work is most closely related to \citet{cai2023uncoupled}.
They establish uncoupled last-iterate convergence along a single
trajectory under a uniform hitting time assumption: under every
stationary policy pair, the expected time to reach any state from any
other state is at most some finite $L$.
Their rate is $\widetilde{\mathcal O}(t^{-1/(9+\nu)})$ for any fixed
$\nu>0$.
This leaves the following question:
\begin{quote}
\emph{Can uncoupled learners achieve a
$\widetilde{\mathcal O}(t^{-1/4})$ last-iterate rate in unknown Markov
games along a single trajectory, under the same hitting time assumption?}
\end{quote}

\paragraph{Our results.}
We answer this question affirmatively with Adaptive Regularized TD
Learning (ARTD; Algorithm~\ref{alg:td-learning}).
Under the uniform hitting time assumption, when both players run ARTD,
for every $\delta\in(0,1)$, with probability at least $1-\delta$,
their current policies satisfy
\[
\operatorname{Gap}(x_t,y_t)
=\widetilde{\mathcal O}\!\left(
\frac{(A+B)(L+1)^{1/4}}
{(1-\gamma)^{9/4}(t+1)^{1/4}}
\right)
\qquad\text{for all }t\ge0,
\]
where $x_t,y_t$ are the policies used at round $t$, $A,B$ are the
players' action counts, and $\gamma$ is the discount factor
(Theorem~\ref{thm:last-iterate}). For fixed game parameters, our algorithm improves the iteration
complexity for reaching an $\varepsilon$-Nash equilibrium from
$\widetilde{\mathcal O}(\varepsilon^{-(9+\nu)})$ to
$\widetilde{\mathcal O}(\varepsilon^{-4})$.
It also improves the hitting time dependence in the duality gap bound
from $L^{2+1/\nu}$ to $(L+1)^{1/4}$.

These guarantees hold uniformly over starting states, and the number of
states enters only through logarithmic factors beyond its effect on $L$.
ARTD requires no knowledge of $L$, the time horizon, or the confidence
level.
Beyond self-play, our algorithm is rational: against an arbitrary
opponent whose policy converges, it almost surely approaches the best
response value for that limiting policy
(Corollary~\ref{cor:rationality}).
Table~\ref{tab:last-iterate-comparison} compares last-iterate guarantees
under bandit feedback.

\begin{table}[t]
\centering
\small
\setlength{\tabcolsep}{4pt}
\renewcommand{\arraystretch}{1.2}
\caption{Last-iterate guarantees under bandit feedback in two-player
zero-sum discounted Markov games. Here $t$ counts transition samples;
fixed game parameters and logarithmic factors are suppressed, and
$\nu>0$ is fixed.
\protect\citet{wei2021markov}, \protect\citet{cai2023uncoupled}, and ARTD
use the same uniform hitting assumption (Assumption~\ref{ass:hitting});
\protect\citet[Assumption~3.2]{chen2024decentralized} additionally require
aperiodicity under every deterministic stationary policy pair.
The first two rates use parameters tuned to the sample budget.
\protect\citet{wei2021markov} also require a hitting time bound, whereas
\protect\citet{cai2023uncoupled} and ARTD do not.
For \protect\citet[Corollary~5]{wei2021markov}, we convert the bound on squared
policy distance to a duality gap bound.
\protect\citet[Theorem~3.5]{chen2024decentralized} bound the expected gap
under a fixed initial distribution; the other guarantees are uniform over
starting states. \emph{All rounds} means simultaneously on one
high-probability event.}
\label{tab:last-iterate-comparison}
\begin{tabular*}{\linewidth}{@{\extracolsep{\fill}}lll@{}}
\toprule
Work & Convergence guarantee & Duality gap bound \\
\midrule
\citet{wei2021markov}
& Final outer iterate; high probability
& $\widetilde{\mathcal O}(t^{-1/8})$ \\
\addlinespace
\citet{chen2024decentralized}
& Final outer iterate; in expectation
& $\widetilde{\mathcal O}(t^{-1/8})$ \\
\addlinespace
\citet{cai2023uncoupled}
& All rounds; high probability
& $\widetilde{\mathcal O}(t^{-1/(9+\nu)})$ \\
\midrule
\textbf{ARTD (this work)}
& \textbf{All rounds; high probability}
& $\widetilde{\mathcal O}(t^{-1/4})$ \\
\bottomrule
\end{tabular*}
\end{table}

\paragraph{Technical ideas.}
ARTD combines log-barrier policy updates with temporal difference (TD)
value estimation and an adaptive regularization schedule.
A key difficulty in \citet{cai2023uncoupled} is tracking equilibria of
regularized local games whose payoffs change with the value estimates,
even between consecutive visits to the same state.
Slowing value updates stabilizes policy learning but also delays value
estimation. Our algorithm separates these two roles.
A fast moving average estimates TD targets, while
clipped corrections update the values used in policy learning only after
every state has supplied a transition.
This limits both the size and frequency of target changes while allowing
faster statistical estimation.

Our analysis further replaces the moving equilibrium comparator with a
reference policy that updates using exact payoffs, so its movement is
controlled directly by the update size.
Log-barrier contraction then controls the played policies' tracking error
together with the TD estimation error.
To make this local control imply convergence in the Markov game, we adapt
the regularization level to the observed TD discrepancy: when value
estimates lag behind, regularization decreases more slowly.
This lets us bound both the local duality gaps and the Bellman residual
throughout learning.
Bellman contraction and the hitting time assumption then convert these
bounds into the $\widetilde{\mathcal O}(t^{-1/4})$ guarantee along the
trajectory.

%% file: prelim.tex
\section{Preliminaries}
\label{sec:prelim}

\paragraph{Notation.}
For a finite set $\mathcal E$, let $\Delta(\mathcal E)$ denote the set of
probability distributions on $\mathcal E$.
For a vector $v\in\mathbb R^{\mathcal S}$, write
$\|v\|_\infty=\max_{s\in\mathcal S}|v(s)|$ and
$\operatorname{span}(v)=\max_{s\in\mathcal S}v(s)-\min_{s\in\mathcal S}v(s)$.
All logarithms are natural, and $\widetilde{\mathcal O}$ suppresses logarithmic
factors.

\paragraph{Discounted zero-sum Markov games.}
We consider a two-player zero-sum Markov game with a finite state set
$\mathcal S$, finite action sets $\mathcal A$ and $\mathcal B$, and discount
factor $\gamma\in[0,1)$.
Write $S=|\mathcal S|$, $A=|\mathcal A|$, and $B=|\mathcal B|$.
At state $s$, the players choose actions $a\in\mathcal A$ and
$b\in\mathcal B$.
The first player's expected reward is $r(s,a,b)\in[0,1]$, and the next state is drawn from
$P(\cdot\mid s,a,b)\in\Delta(\mathcal S)$.
The first player maximizes the discounted reward, while the second player
minimizes it.
Both $r$ and $P$ are unknown; the state and action sets and $\gamma$ are known.

A stationary policy assigns an action distribution to each state.
Let $\Pi_x=\Delta(\mathcal A)^{\mathcal S}$ and
$\Pi_y=\Delta(\mathcal B)^{\mathcal S}$ be the policy spaces of the two players.
For $(x,y)\in\Pi_x\times\Pi_y$, define
\[
V^{x,y}(s)
=\mathbb E_{x,y}\!\left[
  \sum_{k=0}^{\infty}\gamma^k r(s_k,a_k,b_k)
  \,\middle|\,s_0=s
\right],
\]
where $a_k\sim x(\cdot\mid s_k)$ and $b_k\sim y(\cdot\mid s_k)$ are sampled
independently, and states evolve according to $P$.
Thus $V^{x,y}(s)\in\bigl[0,\frac{1}{1-\gamma}\bigr]$.
A pair $(x^\star,y^\star)$ is a Nash equilibrium if, for every state $s$ and
policies $x\in\Pi_x$, $y\in\Pi_y$,
\[
V^{x,y^\star}(s)
\le V^{x^\star,y^\star}(s)
\le V^{x^\star,y}(s).
\]
Such a game admits a stationary Nash equilibrium.
All Nash equilibria have the same value, denoted by $V^\star(s)$.

\paragraph{Uncoupled learning with trajectory feedback.}
Without loss of generality, we fix an initial state $s_0\in\mathcal S$.
The players learn along a single trajectory starting from $s_0$.
At round $t=0,1,\ldots$, they use policies $x_t\in\Pi_x$ and $y_t\in\Pi_y$ and sample $a_t\sim x_t(\cdot\mid s_t)$ and $b_t\sim y_t(\cdot\mid s_t)$ using
independent private randomness.
The environment generates a reward $r_t\in[0,1]$ and a next state
$s_{t+1}$.
Let $\mathcal F_t$ contain the history, the current state, and the policies
before the actions are sampled.
Conditionally on $\mathcal F_t,a_t,b_t$, the feedback satisfies
\[
\mathbb E[r_t\mid\mathcal F_t,a_t,b_t]=r(s_t,a_t,b_t),
\qquad
\mathbb P(s_{t+1}=s'\mid\mathcal F_t,a_t,b_t)
=P(s'\mid s_t,a_t,b_t).
\]
The players receive opposite payoffs $r_t$ and $-r_t$, respectively, so each
can recover $r_t$ from its own payoff.
Each player observes the next state and its own sampled action.

We follow the uncoupled learning setting of \citet{cai2023uncoupled}.
Each player's policy updates use only its own actions, realized payoffs,
and the observed states.
Neither player observes the opponent's actions or policies, and the players
do not communicate or use a shared random seed.

\paragraph{State visitation.}
Following \citet{cai2023uncoupled}, we impose a uniform bound on expected
hitting times.
For a target state $s'$, let
$T_{s'}=\inf\{k\ge0:s_k=s'\}$ for a trajectory initialized at $s_0=s$.

\begin{assumption}[Uniform hitting time]
\label{ass:hitting}
There exists a finite $L\ge0$ such that, for every stationary policy pair
$(x,y)\in\Pi_x\times\Pi_y$ and states $s,s'\in\mathcal S$,
\[
\mathbb E_{x,y}[T_{s'}\mid s_0=s]\le L.
\]
\end{assumption}

Assumption~\ref{ass:hitting} ensures that every state remains reachable
under every stationary policy pair.
The bound $L$ is unknown to the learners.

\paragraph{Last-iterate convergence.}
We measure the quality of a policy pair by its duality gap
\[
\operatorname{Gap}(x,y)
=\max_{s\in\mathcal S}
\left\{
  \max_{x'\in\Pi_x}V^{x',y}(s)
  -\min_{y'\in\Pi_y}V^{x,y'}(s)
\right\}.
\]
For each state, the expression in braces is the sum of the two players'
possible gains from unilateral deviations.
In particular, $\operatorname{Gap}(x,y)=0$ exactly at Nash equilibria.
The policies may change from round to round; at round $t$, we evaluate the
stationary pair $(x_t,y_t)$ used to sample that round's actions.
The maximum over states gives a guarantee uniform over starting states.

Our goal is an anytime algorithm whose updates are independent of the time
horizon and confidence level and which, for every $\delta\in(0,1)$, satisfies
\[
\mathbb P\!\left(
  \operatorname{Gap}(x_t,y_t)\le\varepsilon_t(\delta)
  \text{ for all }t\ge0
\right)\ge1-\delta,
\]
where $\varepsilon_t(\delta)$ decreases to zero at a polynomial rate in $t$.

\paragraph{Rationality.}
We also require each learner to approach a best response whenever the
opponent's policy converges.
For the first player, if $y_t$ converges statewise almost surely to a
stationary policy $y_\infty$, rationality means
\[
\max_{s\in\mathcal S}
\left\{
  \max_{x\in\Pi_x}V^{x,y_\infty}(s)
  -V^{x_t,y_\infty}(s)
\right\}\longrightarrow0
\qquad\text{almost surely}.
\]
The second player's requirement is symmetric.

%% file: algorithm.tex
\section{Algorithm Design}
\label{sec:algorithm}

\citet{cai2023uncoupled} maintain a policy and a value estimate at each
state. Their policies learn regularized local games whose payoffs combine
immediate rewards with estimated future values, while temporal difference
(TD) updates refine the value estimates using observed rewards and
transitions. These estimates both average noisy samples and define the
local games, coupling statistical estimation with the need to keep the
policy learning targets slowly varying. We separate these roles: a fast
moving average estimates the TD target, while bounded value updates
control the motion of the local games. Within these slowly changing games, we use
log barrier regularization to stabilize policy updates from bandit
feedback. Our \emph{Adaptive Regularized TD Learning} (ARTD) algorithm
uses the TD discrepancy to determine how quickly to decrease the
regularization level, keeping policy learning and value estimation at a
common accuracy scale.

\paragraph{From local games to a Markov game equilibrium.}
We first identify what the policies and value estimates need to achieve
to give a small duality gap in the Markov game.
Let $w\in[0,1]^{\mathcal S}$ estimate the normalized equilibrium value
$(1-\gamma)V^\star$. If we use $w$ to evaluate future rewards, the immediate
decision at state $s$ becomes a matrix game with payoff matrix
\[
M_s(w)_{ab}
=(1-\gamma)r(s,a,b)
 +\gamma\sum_{s'\in\mathcal S}P(s'\mid s,a,b)w(s').
\]
The estimate $w(s)$ should agree with the value of the local game
$M_s(w)$. We express this consistency condition through the Bellman operator
\[
(\mathcal T w)(s)
=\max_{u\in\Delta(\mathcal A)}\min_{v\in\Delta(\mathcal B)}
  u^\top M_s(w)v.
\]
The fixed point of $\mathcal T$ is $(1-\gamma)V^\star$, so
$\|\mathcal T w-w\|_\infty$ measures the error in this consistency condition.

The policies must also approximately solve the local games defined by $w$.
Writing $x_s=x(\cdot\mid s)$ and $y_s=y(\cdot\mid s)$, we measure their
accuracy by the largest local duality gap
\[
\Delta(x,y;w)
=\max_{s\in\mathcal S}
\left\{
 \max_{u\in\Delta(\mathcal A)}u^\top M_s(w)y_s
 -\min_{v\in\Delta(\mathcal B)}x_s^\top M_s(w)v
\right\}.
\]
These two errors together control the quality of the current policies:
Bellman contraction gives
\begin{equation}
\label{eq:local-to-global}
\operatorname{Gap}(x,y)
\le \frac{2}{(1-\gamma)^2}\bigl(\Delta(x,y;w)+\|\mathcal T w-w\|_\infty\bigr).
\end{equation}
This bound guides the design: policy updates should reduce the local
duality gaps, while value updates should reduce the Bellman residual.
We now construct these updates from the observed trajectory.

\paragraph{Allowing every state to learn before changing its target.}
We begin by deciding when to change $w$, since each change also modifies
the local games that the policies are learning.
Along a single trajectory, states are visited at different rates.
Updating $w$ after every transition could therefore change the payoff
matrix $M_s(w)$ at a rarely visited state $s$ repeatedly between two
policy updates at that state.
We divide the trajectory into \emph{covers}: each cover ends once every
state has supplied at least one transition sample.
We initialize $w(s)=1/2$ for every state and set the regularization
temperature to $\tau=512(1-\gamma)$.
Here $1-\gamma$ bounds the initial payoff range in each local game,
and $512$ is a sufficiently large numerical constant.
Within a cover, we hold $w$ and $\tau$ fixed and update each policy
whenever its state is visited. We revise $w$ and $\tau$ at the end of
each cover, as described below.
This gives every state at least one policy update under the current
payoff matrix and regularization level before either changes.

\paragraph{Learning a local game from one sample.}
With $w$ fixed within a cover, we can learn each local game from the
transitions observed at its state. At round $t$, the reward and next state
give the normalized TD target
\[
Y_t=(1-\gamma)r_t+\gamma w(s_{t+1}).
\]
Given the history and sampled actions, $Y_t$ has mean
$M_{s_t}(w)_{a_t b_t}$, so it supplies a sample of the local payoff.
We use the TD error $Y_t-w(s_t)$ as the learning signal.
Subtracting $w(s_t)$ preserves payoff differences between actions and gives
the bound
\[
|Y_t-w(s_t)|\le 1-\gamma+\gamma\operatorname{span}(w).
\]
We normalize the feedback by this observable range together with the
temperature. At the start of each cover, set
\[
q=1-\gamma+\gamma\operatorname{span}(w)+\tau,
\qquad \theta=\frac{\tau}{q},
\qquad \varepsilon=\min\{1,(A+B)\theta\}.
\]
Here $\tau$ is the regularization temperature held fixed during the cover.
Adding $\tau$ to the TD error bound ensures that the relative
regularization level $\theta=\tau/q$ lies in $(0,1)$.
The quantity $\varepsilon$ is the corresponding local accuracy scale.
Using the span of $w$ allows the update to adapt to the variation in
continuation values across states.

We next turn this scalar feedback into a direction for updating a policy.
For a player with $d$ actions, indexed by $[d]$, write
$p_s\in\Delta_d:=\Delta([d])$ for its policy at state $s$.
Thus $p=x$ and $d=A$ for the first player, while $p=y$ and $d=B$ for
the second player.
After sampling its own action $i_t\sim p_{s_t}$, the player forms
\begin{equation}
\label{eq:bandit-estimate}
\ell_t=
\begin{cases}
 -(Y_t-w(s_t))/q, & \text{first player},\\
 \phantom{-}(Y_t-w(s_t))/q, & \text{second player},
\end{cases}
\qquad
\widehat g_{t,i}=\frac{\ell_t\mathbf 1\{i=i_t\}}{p_{s_t,i}}.
\end{equation}
This gives an unbiased estimate of each action's expected centered loss. Both players can therefore use the same minimization update.

Small action probabilities amplify the estimate in
\eqref{eq:bandit-estimate}. We use the log barrier regularizer, whose
curvature increases near the boundary of the simplex, together with an
explicit exploration floor $h>0$.
Define
\[
\Psi(u)=-\sum_{i=1}^d\log u_i,
\qquad
\Delta_d(h)=\{u\in\Delta_d:u_i\ge h\text{ for all }i\},
\]
and let
$D_\Psi(u,p)=\Psi(u)-\Psi(p)-\langle\nabla\Psi(p),u-p\rangle$.
The policy update is regularized mirror descent:
\begin{equation}
\label{eq:actor-update}
p_{s_t}^{+}
=\arg\min_{u\in\Delta_d(h)}
\left\{
 \alpha\langle\widehat g_t,u\rangle
 +\alpha\theta\Psi(u)+D_\Psi(u,p_{s_t})
\right\}.
\end{equation}
The linear term responds to the sampled loss, the regularizer stabilizes
learning in the local zero-sum game, and the divergence controls the
change in policy. The floor supplies a quantitative exploration guarantee.
As the temperature decreases, the regularization and exploration biases
decrease with the desired accuracy.

\paragraph{Estimating the value of the current play.}
The policy update addresses the local duality gap in
\eqref{eq:local-to-global}. To reduce the Bellman residual as well, we
need an estimate of $\mathcal T w$.
The TD samples have conditional mean
$\mu_s=x_s^\top M_s(w)y_s$ under the current policies, and
\[
|\mu_s-(\mathcal T w)(s)|\le\Delta(x,y;w).
\]
Thus, as the policies learn their local games, the mean of their observed
TD targets approaches the Bellman target.
We estimate this mean with a moving average $z$, updated using the same
sample as the policy:
\begin{equation}
\label{eq:td-average}
z^{+}(s_t)=(1-\kappa)z(s_t)+\kappa Y_t.
\end{equation}
Since $\mu_s$ changes as the players learn, the averaging rate $\kappa$
must balance sampling noise against the need to track these changes.
We therefore choose the learning rates jointly, with $\beta$ denoting
the step size for value corrections at cover boundaries.
For a sufficiently small universal constant $c>0$, set
\begin{equation}
\label{eq:learning-rates}
h=\frac{c\varepsilon}{A+B},
\qquad \kappa=c\varepsilon^2,
\qquad \alpha=c\theta\kappa,
\qquad \beta=c\theta\alpha.
\end{equation}
The total exploration mass across the two players is at most
$(A+B)h=c\varepsilon$.
For the moving average in \eqref{eq:td-average}, the sampling fluctuation has scale
$q\sqrt\kappa$; choosing $\kappa$ proportional to $\varepsilon^2$ places
this fluctuation at the target scale $(A+B)\tau$.
The smaller actor step $\alpha=c\theta\kappa$ lets $z$ track the changing
mean of the current play.
Finally, $\beta=c\theta\alpha$ keeps value updates slow enough for the
policies to track the resulting changes in the local games.
All four parameters in \eqref{eq:learning-rates} remain fixed throughout a cover.

\paragraph{Updating the value estimate and regularization level.}
At a cover boundary, $z-w$ is the estimated TD discrepancy.
To keep the local games within the actor's tracking accuracy, we bound
each coordinate's value correction by $\beta\tau$.
Let $\operatorname{clip}(v,-\tau,\tau)$ truncate each coordinate of $v$
to the interval $[-\tau,\tau]$. We update
\begin{equation}
\label{eq:critic-update}
w^{+}=w+\beta\operatorname{clip}(z-w,-\tau,\tau).
\end{equation}
Averaging the TD samples first makes the bounded correction follow an
estimate of their mean. Equation~\eqref{eq:critic-update} then guarantees
$\|w^{+}-w\|_\infty\le\beta\tau$, even when the current value discrepancy
is large.

The regularization level should decrease slowly enough for value
learning to keep up. Since the value update clips each coordinate of
$z-w$ at magnitude $\tau$, we use the largest discrepancy
$\|z-w\|_\infty$ to determine how quickly to reduce $\tau$:
\begin{equation}
\label{eq:temperature-update}
\tau^{+}
=\tau\left(
1-\frac{\beta}{2\bigl(\frac{1}{1-\gamma}+\sqrt{A+B}\bigr)}
       \frac{\tau}{\max\{\tau,\|z-w\|_\infty\}}
\right).
\end{equation}
When $\|z-w\|_\infty$ is large relative to $\tau$, we decrease $\tau$
more slowly, giving the value estimate more time to improve.
As this relative discrepancy shrinks, we decrease $\tau$ by a larger
percentage.
The factor $\frac{1}{1-\gamma}+\sqrt{A+B}$ accounts for how discounting affects the progress
of Bellman updates and how changes in temperature affect the regularized
local games.
Once $z$ approximates the Bellman target, a sufficiently large Bellman residual receives a correction
of size $\beta\tau$, while its Bellman target moves by at most
$\gamma\beta\tau$. The resulting progress is $(1-\gamma)\beta\tau$.
The factor $1/(1-\gamma)$ in \eqref{eq:temperature-update} slows the
decrease in $\tau$ accordingly, giving the Bellman residual time to
reach the smaller target error.

\paragraph{Pseudocode and implementation.}
Algorithm~\ref{alg:td-learning} gives the pseudocode of ARTD for either player.
In addition to the initialization of $w$ and $\tau$ above, we initialize
every policy uniformly and set $z(s)=1/2$ for every state.
Both players compute $z,w,\tau$ from the same observed states and
recovered rewards $r_t$, so the common initialization gives identical
estimates throughout learning.
The rates are determined by the current estimates and known game
dimensions, making the algorithm anytime and independent of $L$ and
the confidence level.

\begin{algobox}{Adaptive Regularized TD Learning (ARTD)}
\label{alg:td-learning}
\textbf{Input:} state set $\mathcal S$, action counts $A,B$, discount
$\gamma$, and the player index (first or second).\\
\textbf{Initialize:} Let $d$ be the player's
action count. Fix a sufficiently small universal constant $c>0$.
For every state $s$, set $p_s=(1/d,\ldots,1/d)$ and $z(s)=w(s)=1/2$.
Set $\tau=512(1-\gamma)$, $t=0$, and observe $s_0$.
\begin{enumerate}
\setlength{\itemsep}{3pt}
\setlength{\parskip}{0pt}
\item \textbf{Repeat for each cover:}
\begin{enumerate}
\setlength{\itemsep}{3pt}
\item Set $q=1-\gamma+\gamma\operatorname{span}(w)+\tau$,
$\theta=\tau/q$, and $\varepsilon=\min\{1,(A+B)\theta\}$.
Set $h=c\varepsilon/(A+B)$, $\kappa=c\varepsilon^2$,
$\alpha=c\theta\kappa$, and $\beta=c\theta\alpha$.
Initialize the set of covered states $\mathcal C=\varnothing$.
\item \textbf{While $\mathcal C\ne\mathcal S$:}
\begin{enumerate}
\setlength{\itemsep}{3pt}
\item Sample own action $i_t\sim p_{s_t}$; observe own payoff and
$s_{t+1}$, and recover $r_t$.
\item Set $Y_t=(1-\gamma)r_t+\gamma w(s_{t+1})$ and
\[
\ell_t=
\begin{cases}
 -(Y_t-w(s_t))/q, & \text{first player},\\
 \phantom{-}(Y_t-w(s_t))/q, & \text{second player}.
\end{cases}
\]
Form $\widehat g_{t,i}=\ell_t\mathbf 1\{i=i_t\}/p_{s_t,i}$
using the sampling probabilities.
\item Update only the visited state's policy and TD estimate:
\[
\begin{split}
p_{s_t}&\leftarrow
\arg\min_{u\in\Delta_d(h)}
\{\alpha\langle\widehat g_t,u\rangle
  +\alpha\theta\Psi(u)+D_\Psi(u,p_{s_t})\},\\
z(s_t)&\leftarrow(1-\kappa)z(s_t)+\kappa Y_t.
\end{split}
\]
\item Set $\mathcal C\leftarrow\mathcal C\cup\{s_t\}$ and
$t\leftarrow t+1$.
\end{enumerate}
\item Using the current $z$ and the old $w,\tau$, compute simultaneously
\[
\begin{split}
w^{+}&=w+\beta\operatorname{clip}(z-w,-\tau,\tau),\\
\tau^{+}&=\tau\left(1-
\frac{\beta\tau}{2\bigl(\frac{1}{1-\gamma}+\sqrt{A+B}\bigr)\max\{\tau,\|z-w\|_\infty\}}
\right).
\end{split}
\]
Set $(w,\tau)\leftarrow(w^{+},\tau^{+})$.
\end{enumerate}
\end{enumerate}
\end{algobox}

The policy update in \eqref{eq:actor-update} can be computed by solving
a scalar equation. For the policy $p=p_{s_t}$ before the update, set
\[
v_i=\frac{1+\alpha\ell_t\mathbf 1\{i=i_t\}}
           {(1+\alpha\theta)p_i}.
\]
The solution is
\[
p_i^{+}=\max\{h,(v_i+\lambda)^{-1}\},
\qquad \sum_{i=1}^d p_i^{+}=1,
\]
where the scalar $\lambda>-\min_i v_i$ can be found by bisection.
Over $T$ rounds, each player uses $\mathcal O(Sd+Tdm)$ time with $m$
bisection iterations per policy update.
Initializing the policies costs $\mathcal O(Sd)$, and each bisection
iteration costs $\mathcal O(d)$.
The value and temperature updates cost $\mathcal O(S)$ per cover;
since each cover contains at least $S$ rounds, their total cost is
$\mathcal O(T)$.
\FloatBarrier

%% file: theory.tex
\section{Theoretical Guarantees}
\label{sec:theory}

We first give a last-iterate guarantee when both players run ARTD.
We then establish rationality for either
player facing an opponent whose policy converges.

\begin{theorem}[Last-iterate convergence]
\label{thm:last-iterate}
Suppose Assumption~\ref{ass:hitting} holds and both players run
ARTD (Algorithm~\ref{alg:td-learning}) with a sufficiently small universal
constant $c>0$. Then, for every
$\delta\in(0,1)$, with probability at least $1-\delta$,
simultaneously for every integer $t\ge0$,
\begin{equation}
\label{eq:main-rate}
\operatorname{Gap}(x_t,y_t)
=\widetilde{\mathcal O}\!\left(
\frac{(A+B)(L+1)^{1/4}}
{(1-\gamma)^{9/4}(t+1)^{1/4}}
\right).
\end{equation}
Here $\widetilde{\mathcal O}$ hides logarithmic factors in
$S,t+2,1/\delta$, and $1/(1-\gamma)$.
\end{theorem}
Our algorithm ARTD improves the last-iterate convergence rate from
$\widetilde{\mathcal O}(t^{-1/(9+\nu)})$ of
\citet[Theorem~2]{cai2023uncoupled}, for any fixed $\nu>0$, to
$\widetilde{\mathcal O}(t^{-1/4})$ under the same bandit feedback model
and uniform hitting assumption. For fixed game parameters, this reduces
the iteration complexity for reaching an $\varepsilon$-Nash equilibrium
from $\widetilde{\mathcal O}(\varepsilon^{-(9+\nu)})$ to
$\widetilde{\mathcal O}(\varepsilon^{-4})$.
ARTD also improves the dependence on the hitting time in the duality gap
bound from $L^{2+1/\nu}$ to $(L+1)^{1/4}$.
We present a simplified parameter dependence in
Theorem~\ref{thm:last-iterate} and give the sharper bound in
Appendix~\ref{app:rate}.

Our algorithm achieves these improvements with uncoupled updates along
a single trajectory, using only each player's own actions, realized
payoffs, and observed states. It requires no knowledge of $L$, the time
horizon, or the confidence level. The guarantee holds simultaneously at
all rounds for the policies actually played, uniformly over starting
states. Beyond its effect on $L$, the state count $S$ contributes only
logarithmic factors.

Theorem~\ref{thm:last-iterate} concerns self-play, where both players run
ARTD. Our analysis also shows that a single player running ARTD
asymptotically attains the best response value against any opponent
whose policy converges.

\begin{corollary}[Rationality]
\label{cor:rationality}
Under Assumption~\ref{ass:hitting}, either player running ARTD with $c$ as in
Theorem~\ref{thm:last-iterate} is rational in the sense of
Section~\ref{sec:prelim}. In particular, if $y_t$ converges statewise almost
surely to a stationary policy $y_\infty$, then
\[
\max_{s\in\mathcal S}
\left\{\max_{x\in\Pi_x}V^{x,y_\infty}(s)
-V^{x_t,y_\infty}(s)\right\}\longrightarrow0
\qquad\text{almost surely}.
\]
The symmetric statement holds for the second player.
\end{corollary}
Corollary~\ref{cor:rationality} ensures that our algorithm learns to
exploit suboptimality in the limiting policy of an arbitrary opponent.

\paragraph{Proof sketch.}
The main difficulty is that policy learning and value estimation change
each other's targets. We compare the sampled policies with a reference
policy driven by exact local payoffs, whose local duality gap is
$\mathcal O((A+B)\tau)$. Bounded value movement controls the reference
drift, and zero-sum cancellation gives contraction under the log barrier.
Together with faster TD averaging, this bounds both the local duality gap
and the target estimation error by $\widetilde{\mathcal O}((A+B)\tau)$.

Clipped value updates then reduce the Bellman residual. When a
coordinate's residual is large, its value moves by $\beta\tau$ toward
the target, while the Bellman target moves by at most $\gamma\beta\tau$.
The adaptive temperature update preserves this net progress of
$(1-\gamma)\beta\tau$ as the required accuracy tightens, yielding
\[
\operatorname{Gap}(x_t,y_t)
=\widetilde{\mathcal O}\!\left(
\frac{(A+B)\tau_t}{(1-\gamma)^2}\right).
\]
Here $\tau_t$ is the temperature before round $t$.

Finally, the hitting assumption bounds each cover by
$\widetilde{\mathcal O}(L+1)$ rounds. At small temperatures,
$\beta=c^3(A+B)^2(\tau/q)^4$, so the number of covers needed to reach
temperature $\tau$ scales as $\tau^{-4}$.
Converting covers to rounds and inverting this count gives the
$t^{-1/4}$ rate in~\eqref{eq:main-rate}.
Appendix~\ref{app:last-iterate} gives the full proof.

For rationality, we track reciprocal action probabilities to establish
asymptotic greediness and convergence to the best response value;
see Appendix~\ref{app:rationality}.

%% file: related_work.tex
\section{Related Work}
\label{sec:related-work}

\paragraph{Last-iterate convergence.}
With full gradient feedback, optimistic and extragradient methods have
been extensively studied for last-iterate convergence in matrix games,
saddle point optimization, and monotone variational inequalities
\citep{daskalakis2019last,mertikopoulos2019optimistic,golowich2020last,wei2021linear,cai2022tight}.
Under bandit feedback, recent work develops uncoupled learning dynamics
for zero-sum games, monotone games, and bilinear saddle point problems
\citep{cai2023uncoupled,dong2026uncoupled,maiti2026efficient}.
A complementary approach converts average-iterate guarantees into
last-iterate guarantees through a black-box reduction
\citep{cai2026average}.
For uncoupled learning in matrix games with bandit feedback,
\citet{fiegel2026harder} establish a $t^{-1/4}$ lower bound for
$L^2$ last-iterate convergence, and
\citet{fiegel2026optimal} obtain a matching
$\widetilde{\mathcal O}(t^{-1/4})$ high-probability guarantee using
log-barrier regularization.
Our algorithm also uses the log barrier, but must jointly control policy
learning and changing continuation values from a single trajectory.
When opponent actions are additionally observed, faster last-iterate
rates are possible \citep{hait2026near,zhang2026minimax}.
Our two players learn independently, without observing each other's
actions or communicating with each other.

\paragraph{Markov games.}
Markov games provide a standard framework for sequential multi-agent
decision making \citep{shapley1953stochastic,littman1994markov}.
A substantial literature studies sample-efficient equilibrium learning
through online exploration and self-play
\citep{bai2020provable,bai2020near,liu2021sharp,jin2024vlearning},
as well as learning with a generative model
\citep{zhang2020model}.
Offline equilibrium learning instead uses a fixed dataset, with
guarantees governed by coverage and function approximation
\citep{cui2022offline,zhong2022pessimistic,zhang2023offline}.
For policy optimization with full information, entropy regularization
and optimistic updates yield last-iterate guarantees in zero-sum
Markov games \citep{cen2021fast,cen2023faster}.
Another line studies decentralized learning dynamics, including
asymptotic convergence and rationality, optimistic policy updates, and
finite-sample analyses of learning based on best responses
\citep{sayin2021decentralized,wei2021markov,chen2024decentralized}.
Most directly related to our setting, \citet{cai2023uncoupled} establish
last-iterate convergence for uncoupled learning along a single
trajectory with bandit feedback.
Under the same feedback model and uniform hitting assumption, our
algorithm ARTD improves their
$\widetilde{\mathcal O}(t^{-1/(9+\nu)})$ rate, for any fixed $\nu>0$,
to $\widetilde{\mathcal O}(t^{-1/4})$, while requiring no knowledge of
the hitting time bound, the time horizon, or the confidence level.

%% file: conclusion.tex
\section{Conclusion}
\label{sec:conclusion}

We study uncoupled last-iterate learning in unknown two-player zero-sum
discounted Markov games with bandit feedback.
Under the uniform hitting time assumption, our algorithm ARTD achieves a
$\widetilde{\mathcal O}(t^{-1/4})$ duality gap bound for the current
policies, with high probability simultaneously over all rounds and
starting states.
This improves the $\widetilde{\mathcal O}(t^{-1/(9+\nu)})$ rate of
\citet{cai2023uncoupled}, for any fixed $\nu>0$, under the same feedback
model and hitting time assumption.
ARTD learns along a single trajectory and requires no knowledge of the
hitting time bound, the time horizon, or the confidence level.
Our design separates fast temporal difference averaging from bounded
value updates, allowing statistical estimation to proceed while limiting
changes in the targets used for policy learning.
We control the resulting policy error by tracking a reference policy
and adapt regularization to the progress of value estimation.
By reconciling fast value estimation with stable policy learning, ARTD
achieves the improved last-iterate rate within the same uncoupled learning
model.

%% file: proofs.tex
\section{Proof of Theorem~\ref{thm:last-iterate}}
\label{app:last-iterate}

We first control learning within each state, then establish Bellman
consistency, and finally count interactions. Throughout this appendix,
$C$ denotes a positive numerical constant that may increase between
displays. Once the universal algorithm constant $c$ is fixed sufficiently
small, these constants are independent of $S,A,B,L,\gamma,t$, and
$\delta$. Subscripts $t$ refer to quantities just before round $t$;
a superscript $+$ denotes the result of the particular update under
discussion. At a cover boundary the target estimate $z$ includes the
sample that completed the cover.
We write $\Delta_t=\Delta(x_t,y_t;w_t)$ and $[a]_+=\max\{a,0\}$.
At a fixed state's visits, $x_n,y_n$ denote its action distributions
and $z_n$ its target estimate; $w_n$ remains the full value vector.
For the probability bounds, write
\[
\Lambda_t=\log\frac{8S(t+2)^2}{\delta}.
\]

\subsection{Pathwise bounds and Bellman certificates}
\label{app:pathwise}

It is useful to express the composite policy update as an ordinary mirror
step. Define its effective step size and contraction scale by
\[
\eta=\frac{\alpha}{q(1+\alpha\theta)},
\qquad v=\eta\tau=\frac{\alpha\theta}{1+\alpha\theta}.
\]
If $\widehat F=q\widehat g$ is the importance weighted centered loss on
the original scale, expanding the objective
in~\eqref{eq:actor-update} and dividing by $1+\alpha\theta$ gives exactly
\begin{equation}
\label{eq:explicit-mirror}
p^+=\arg\min_{u\in\Delta_d(h)}
\left\{D_\Psi(u,p)+\eta\langle
\widehat F+\tau\nabla\Psi(p),u\rangle\right\}.
\end{equation}
In particular, $\beta/v=c(1+\alpha\theta)\le2c$.

\begin{lemma}[Movement and scale bounds]
\label{lem:movement}
All value estimates belong to $[0,1]^{\mathcal S}$, and
\[
0<\tau\le512(1-\gamma),\qquad
1-\gamma\le q\le513,\qquad \tau\le q.
\]
At a cover boundary,
\begin{equation}
\label{eq:boundary-budgets}
\|w^+-w\|_\infty\le\beta\tau,\qquad
0\le\tau-\tau^+\le
\frac{\beta\tau}{2((1-\gamma)^{-1}+\sqrt{A+B})},\qquad
|q^+-q|\le3\beta\tau.
\end{equation}
Adjacent scales and gains differ by a factor $1+\mathcal O(\beta)$.
At a visit, either player's policy satisfies
\begin{equation}
\label{eq:policy-movement}
\|p^+-p\|_1\le C\eta q+C c\beta\theta.
\end{equation}
Here the second term also covers an increase of the floor since the
previous visit; the current and preceding $\beta$ are comparable.
When both players use the prescribed updates, the conditional means
$\mu=x^\top M_s(w)y$ at consecutive visits satisfy
\begin{equation}
\label{eq:mean-movement}
|\mu_{n+1}-\mu_n|\le C\eta_n q_n^2+C\beta_n\tau_n.
\end{equation}
The same bound holds over a partial interval after a visit.
\end{lemma}

\begin{proof}
The target $Y=(1-\gamma)r+\gamma w(s')$ belongs to $[0,1]$.
Both the averaging update and the clipped value update are convex
combinations of numbers in this interval. The temperature is positive
and decreasing. These observations prove the first assertions and
the first two bounds in~\eqref{eq:boundary-budgets}. The span changes by
at most $2\|w^+-w\|_\infty$, proving the third bound. Since $q\ge\tau$,
all relative changes are $\mathcal O(\beta)$; the minimum in
$\varepsilon=\min\{1,(A+B)\theta\}$ preserves this comparison.

We verify the policy movement bound even when the old policy satisfies
only the preceding floor. Put $Z=\pm(Y-w(s))$, so $|Z|\le q$, and let
$i_*$ be the sampled action. The optimality conditions
of~\eqref{eq:explicit-mirror} give
\[
\frac1{p_i^+}=\frac{1-v+\eta Z\mathbf1\{i=i_*\}}{p_i}
+\lambda-\nu_i,
\qquad \nu_i\ge0,\quad \nu_i(p_i^+-h)=0.
\]
Let $h_-$ be the floor satisfied by $p$. If $\lambda\ge0$, every
increasing free coordinate has denominator at least
$1-\eta(\tau+|Z|)\ge1/2$, and its increase is at most
$2\eta(\tau+|Z|)p_i$. Increasing constrained coordinates contribute at
most $d(h-h_-)_+$ in total. If $\lambda<0$, every unplayed free
coordinate increases. A constrained unplayed coordinate also cannot
decrease: its optimality inequality is
$1-v+\lambda p_i\ge p_i/h$, which implies $h>p_i$.
Thus only the sampled coordinate can decrease, by at most $\eta|Z|$.
Twice the total increase or decrease yields
\[
\|p^+-p\|_1\le4\eta(\tau+|Z|)+2d(h-h_-)_+.
\]
To bound the floor change, write $R=q-\tau$. At a boundary,
$\Delta\tau\le0$ and $|\Delta R|\le2\beta\tau$. Consequently
\[
(\theta^+-\theta)_+
=\left[\frac{R\Delta\tau-\tau\Delta R}{q q^+}\right]_+
\le C\beta\theta^2.
\]
Since $h=c\min\{\theta,1/(A+B)\}$ and $d\le A+B$, this implies
$d(h^+-h)_+\le Cc\beta\theta$. For a step crossing the point
$\theta=1/(A+B)$, the same conclusion follows by truncating the
increment at that point. This proves~\eqref{eq:policy-movement}.

There is at most one cover boundary between consecutive visits to a
state. After its visit it must be visited again before a second cover
can end. Changes in either policy alter $x^\top M_s(w)y$ by at most
the matrix's entry range times the policy's $\ell_1$ change.
That range is at most $q$. A change in $w$ contributes at most
$\gamma\|\Delta w\|_\infty$. Applying the preceding bounds gives
\eqref{eq:mean-movement}, including the boundary immediately after
a closing sample.
\end{proof}

\begin{lemma}[From local accuracy to the duality gap]
\label{lem:certificate}
For any stationary pair $(x,y)$ and $w\in[0,1]^{\mathcal S}$,
\[
\operatorname{Gap}(x,y)\le
\frac{2}{(1-\gamma)^2}
\left(\Delta(x,y;w)+\|\mathcal T w-w\|_\infty\right).
\]
Moreover, Assumption~\ref{ass:hitting} implies
\begin{equation}
\label{eq:span-residual}
\operatorname{span}(w)\le
L\bigl(1-\gamma+\operatorname{span}(\mathcal T w-w)\bigr).
\end{equation}
\end{lemma}

\begin{proof}
Let $\mathcal T_y$ and $\mathcal T_x$ be the normalized best response
Bellman operators with $y$ and $x$ fixed, respectively. Their values at
$w$ enclose $\mathcal T w$, and
\[
\mathcal T_y w\le w+
(\|\mathcal T w-w\|_\infty+\Delta)\mathbf1,
\qquad
\mathcal T_x w\ge w-
(\|\mathcal T w-w\|_\infty+\Delta)\mathbf1.
\]
Monotonicity, translation by $\gamma$ times a constant, and contraction
bound their fixed points above and below by $w$ plus or minus
$(\|\mathcal T w-w\|_\infty+\Delta)/(1-\gamma)$.
These fixed points are normalized values; division by $1-\gamma$
proves the first claim.

For the span claim, choose a saddle pair for every $M_s(w)$ and fix
the resulting stationary policies. With their transition matrix $P_0$
and normalized reward vector $g\in[0,1-\gamma]^{\mathcal S}$,
\[
w=f+\gamma P_0w,
\qquad f=g-(\mathcal T w-w).
\]
Subtracting $\min f/(1-\gamma)$ from $w$ gives a vector $\widetilde w$
that is the discounted value of $f-\min f\ge0$.
Start at a state maximizing $\widetilde w$ and stop upon hitting a state
minimizing it. The discounted stopping identity and
$\widetilde w\ge0$ give
\[
\operatorname{span}(w)
\le \mathbb E\left[\sum_{j<T_{s'}}\gamma^j\right]
\operatorname{span}(f)
\le L\operatorname{span}(f).
\]
The stopping identity follows first for $T_{s'}\wedge m$ and then by
bounded convergence. Finally
$\operatorname{span}(f)\le1-\gamma+
\operatorname{span}(\mathcal T w-w)$.
\end{proof}

\subsection{A reference policy for the changing local games}
\label{app:reference}

We need a comparator that moves slowly enough for a sampled policy to
track it. Following exact regularized equilibria would require controlling
their sensitivity to changes in $w$. We instead let the reference take
one exact gradient step whenever its state is visited, so its movement
is controlled directly by the step size.

For one state, write $N=M_s(w)-w(s)\mathbf1\mathbf1^\top$ and
$p=(x,y)$. The centered game operator is
\[
F(p)=(-Ny,N^\top x)=Jp,\qquad J^\top=-J.
\]
Its coordinates have absolute value at most $q$, and
$\langle F(p)-F(u),p-u\rangle=0$.
For a matrix $M$, write
$\operatorname{gap}_M(x,y)=\max_u u^\top My-\min_v x^\top Mv$,
where the extrema range over the two action simplices.
Adding a constant to a block of $F$ changes neither a policy update
nor any residual below. On each simplex block define
\begin{equation}
\label{eq:reference-residual}
b_p=\frac{\sum_i p_i^2F_i(p)-\tau}{\sum_i p_i^2},
\qquad g_i(p)=p_i(F_i(p)-b_p)-\tau.
\end{equation}
Thus $\sum_i p_i g_i=0$. The norm $\|g(p)\|_2$ is the log-barrier
dual residual of \citet[Definition~6.3]{fiegel2026optimal}, written
explicitly for the two simplices. We use $\Psi$ for the sum of the two barriers.
The reference $\bar p$ starts uniformly. At a visit it takes the
unconstrained mirror step
\[
\bar p^+=\arg\min_{u\in\Delta(\mathcal A)\times\Delta(\mathcal B)}
\{D_\Psi(u,\bar p)+
\eta\langle F(\bar p)+\tau\nabla\Psi(\bar p),u\rangle\}.
\]
The minimizer is interior. This reference is used only in the proof.

\begin{lemma}[Reference accuracy and relative movement]
\label{lem:reference}
For sufficiently small universal $c$, the reference satisfies, at all
times and also immediately after its updates,
\[
\|g(\bar p)\|_2\le\frac{\tau}{64},\qquad
\operatorname{gap}_{M_s(w)}(\bar x,\bar y)\le2(A+B)\tau.
\]
At a visit,
\[
\max_i\left|\frac{\bar p_i^+}{\bar p_i}-1\right|\le\frac{v}{16}.
\]
Both $\bar p$ and $\bar p^+$ are feasible for the current exploration
floor, with a fixed multiplicative margin.
\end{lemma}

\begin{proof}
We first keep $M$ and $\tau$ fixed. Let
$r_i=(p_i^+-p_i)/p_i$ for the exact step, and let
$\mathcal R=\|g(p)\|_2$. Its optimality conditions and
$\sum_i(p_i^+-p_i)=0$ imply
\[
\sum_i[r_i-\log(1+r_i)]\le-\eta\langle g,r\rangle.
\]
The left side is at least
$\|r\|_2^2/[2(1+\|r\|_2)]$. Hence, if
$\eta\mathcal R\le1/8$,
\begin{equation}
\label{eq:relative-expansion}
\|r\|_2\le4\eta\mathcal R,
\qquad r=-\eta g+e,
\qquad \|e\|_2\le C\eta^2\mathcal R^2.
\end{equation}
To verify the remainder without a dimension factor, write the block
optimality equation as
$p_i^+=p_i/(1+\eta(g_i+\ell p_i))$. With
$Q_p=\sum_i p_i^2$ and $H_i=g_i+\ell p_i$, normalization gives
\[
\ell Q_p=\eta\sum_i\frac{p_iH_i^2}{1+\eta H_i}\ge0.
\]
The first bound on $r$ gives $\|H\|_2\le C\mathcal R$.
Since $\max_i p_i\le\sqrt{Q_p}$, it follows that
$\ell\sqrt{Q_p}\le C\eta\mathcal R^2$.
Expanding $-\eta H_i/(1+\eta H_i)$ now proves
the last bound in~\eqref{eq:relative-expansion}.

Let $P_p$ denote blockwise orthogonal projection onto the subspace
orthogonal to $p$ in each block. Then
$g=P_p(\operatorname{diag}(p)F(p)-\tau\mathbf1)$.
Subtracting the old block normalizer before projecting at $p^+$ gives
the exact identity
\[
g(p^+)=P_{p^+}\left[g+\tau r+r\odot g+
\operatorname{diag}(p^+)J\operatorname{diag}(p)r\right].
\]
The matrix $C_p=\operatorname{diag}(p)J\operatorname{diag}(p)$ is
skew-symmetric and has operator norm at most $q$: its off-diagonal
block has Frobenius norm at most
$q\|x\|_2\|y\|_2\le q$.
Substitute~\eqref{eq:relative-expansion}, use the nonexpansiveness of
$P_{p^+}$, and cancel $\langle g,C_pg\rangle=0$. This gives
\begin{equation}
\label{eq:residual-contraction}
\|g(p^+)\|_2^2
\le\bigl(1-2\eta\tau+C\eta^2(q+\tau+\mathcal R)^2\bigr)
\mathcal R^2+2\eta\mathcal R^3.
\end{equation}
For clarity, the vector before projection is
$g-\eta(\tau g+g\odot g+C_pg)+E$ with
$\|E\|_2\le C\eta^2\mathcal R^2(q+\tau+\mathcal R)$;
squaring this expression proves the displayed inequality.
The gains satisfy
\[
\eta q\le c^2,\qquad \frac{\eta q^2}{\tau}\le c^2.
\]
Thus inside $\mathcal R\le\tau/64$, choosing $c$ small gives
from~\eqref{eq:residual-contraction} that
$\|g(p^+)\|_2^2\le(1-v)\mathcal R^2$.

A change of the matrix by at most $u$ in every entry changes $g$ by
at most $\sqrt2u$, because each block's vector
$\operatorname{diag}(p)\Delta F$ has norm at most $u$.
A temperature change changes it by at most
$\sqrt{A+B}|\Delta\tau|$. Constant matrix shifts have no effect.
Between two visits,~\eqref{eq:boundary-budgets} therefore changes the
residual norm by at most $C\beta\tau$. Since $\beta\le2cv$,
the bound at the next visit is at most
$\sqrt{1-v}\,\tau/64+C\beta\tau$. For small enough $c$, this is
at most $\tau(1-\beta/[2((1-\gamma)^{-1}+\sqrt{A+B})])/64$,
which is no larger than the next radius $\tau^+/64$. Initially, in a uniform block,
$g_i=(F_i-\overline F)/d$, so
\[
\|g\|_2^2\le
(1-\gamma)^2\left(\frac1{4A}+\frac1{4B}\right)
\le\frac{(1-\gamma)^2}{2}.
\]
Since $\tau_0=512(1-\gamma)$, induction proves the residual bound.
The first inequality of~\eqref{eq:relative-expansion} proves the
relative movement bound.

For feasibility, consider a block of size $d$. Its loss coordinates
have range at most $q$. From~\eqref{eq:reference-residual},
\[
\tau+g_i=p_i(F_i-b_p),\qquad
\sum_i p_i(F_i-b_p)=d\tau+\sum_i g_i.
\]
The second identity and $|\sum_i g_i|\le\sqrt d\,\tau/64$ show
$F_i-b_p\le q+(d+\sqrt d/64)\tau$. Hence
\[
p_i\ge\frac{(1-1/64)\tau}
{q+(d+\sqrt d/64)\tau}
\ge c_1\min\{\tau/q,1/(A+B)\}
\]
for a universal $c_1>0$. This exceeds $4h$ when $c$ is small;
the relative movement preserves feasibility after the step.
Finally $F_i-b_p>0$, and
\[
\langle p,F\rangle-\min_iF_i
\le d\tau+\sum_i g_i.
\]
Adding the two players' regrets proves the local gap bound, as in
\citet[Lemma~6.4]{fiegel2026optimal}.
\end{proof}

\subsection{A common concentration argument for policies and targets}
\label{app:tracking}

We first record the probability argument. Its gains can depend on the
entire past; this is needed because the observed trajectory sets the
cover boundaries and the temperature.

\begin{lemma}[Stochastic contraction]
\label{lem:stochastic}
Let $X_n\ge0$ be adapted to a filtration $(\mathcal G_n)$, with
$\mathcal G_n$-measurable $u_n\in[0,1]$. Suppose
\[
X_{n+1}\le(1-c_1u_n)X_n+C_1u_n+W_n,
\qquad \mathbb E[W_n\mid\mathcal G_n]=0,
\]
and, for $0\le\lambda\le\lambda_0$,
\[
\log\mathbb E[e^{\lambda W_n}\mid\mathcal G_n]
\le C_2\lambda^2u_n(X_n+1).
\]
Here $c_1,C_1,C_2,\lambda_0>0$ are numerical constants and $X_0$ is
bounded by a numerical constant. For any collection of $m$ such
processes and any $\delta\in(0,1)$, with probability at least
$1-\delta$, all of them satisfy
\[
X_n\le C\log\frac{Cm(n+2)^2}{\delta}
\qquad(n\ge0).
\]
\end{lemma}

\begin{proof}
Choose a fixed $\lambda>0$ small enough that
$C_2\lambda\le c_1/2$. The assumptions imply
\[
\mathbb E[e^{\lambda X_{n+1}}\mid\mathcal G_n]
\le \exp\{\lambda X_n+
\lambda u_n(C_3-c_1X_n/2)\}.
\]
For $X_n\ge2C_3/c_1$ the right side is at most
$e^{\lambda X_n}$. On the remaining bounded interval its excess over
$e^{\lambda X_n}$ is at most $C u_n$. Taking expectations gives
$\mathbb E e^{\lambda X_n}\le C(n+1)$.
Markov's inequality at threshold
$\lambda^{-1}\log(Cm(n+2)^3/\delta)$, followed by a union bound over
all $n$ and all processes, proves the claim. The cube inside this
logarithm is absorbed by a numerical change in $C$.
\end{proof}

\begin{lemma}[Local learning and target estimation]
\label{lem:local-learning}
When both players run the algorithm, for every $\delta\in(0,1)$,
with probability at least $1-\delta$, at every interaction time,
\begin{equation}
\label{eq:local-tracking}
\Delta_t\le C(A+B)\tau_t\sqrt{\Lambda_t},
\qquad
\max_s|z_t(s)-x_t(\cdot\mid s)^\top M_s(w_t)y_t(\cdot\mid s)|
\le C(A+B)\tau_t\sqrt{\Lambda_t}.
\end{equation}
The same bounds hold just after the policy and target updates of a
closing round, before its value correction.
\end{lemma}

\begin{proof}
We give the policy and target recursions separately, then apply
Lemma~\ref{lem:stochastic} to both. Fix a state and index its visits
by $n=0,1,\ldots$. Let $\mathcal G_n$ be the full history immediately
before the actions at its $n$th visit. Every reference step is
$\mathcal G_n$-measurable.

\paragraph{Policy divergence.}
Write $\mathcal D=D_\Psi(\bar p,p)$, summing over the two blocks,
and set $\mathcal U=\mathcal D+D_\Psi(p,\bar p)$.
The three point inequality for~\eqref{eq:explicit-mirror}, tested
with the feasible $\bar p^+$, gives
\[
D_\Psi(\bar p^+,p^+)
\le D_\Psi(\bar p^+,p)
+\eta\langle F(p)+\tau\nabla\Psi(p),\bar p^+-p\rangle
+C\eta^2(q^2+(A+B)\tau^2)+\zeta,
\]
where
$\zeta=\eta\langle\widehat F-F(p),\bar p^+-p\rangle$ is centered
conditionally on $\mathcal G_n$.
To check the stability term globally, put
$b_i=\eta p_i(\widehat F_i-\tau/p_i)$.
Then $|b_i|\le\eta(q+\tau)\le1/2$, and
\[
\sup_{r>-1}\{-b_i r-r+\log(1+r)\}
=b_i-\log(1+b_i)\le Cb_i^2.
\]
Summing gives the asserted stability bound. This argument only needs
the comparator to be feasible; the old policy can satisfy the preceding
floor.

Here is the cancellation that supplies contraction. Put
$d_i=(\bar p_i^+-\bar p_i)/\bar p_i$ and
$\mathcal U_i=(\bar p_i-p_i)^2/(\bar p_i p_i)$.
The reference optimality conditions and skew-symmetry give the exact
identity
\begin{align*}
&\eta\langle F(p)+\tau\nabla\Psi(p),\bar p-p\rangle
+\langle\nabla\Psi(\bar p)-\nabla\Psi(p),\bar p^+-\bar p\rangle\\
&\qquad=-v\mathcal U
+\sum_i\frac{d_i}{1+d_i}\mathcal U_i
+\sum_i\frac{d_i^2}{1+d_i}
\left(\frac{\bar p_i}{p_i}-1\right).
\end{align*}
For example, substitute
$\eta F_i(\bar p)=v/\bar p_i-
d_i/[\bar p_i(1+d_i)]$ up to a block constant; the block constants
vanish against $\bar p-p$. The remaining terms follow by expanding
$D_\Psi(\bar p^+,p)-D_\Psi(\bar p,p)$.
Use $|d_i|\le v/16$ and
\[
\sum_i\left|\frac{\bar p_i}{p_i}-1\right|
\le \mathcal U+(A+B),\qquad
\sum_i\frac{\bar p_i}{p_i}\le\mathcal U+2(A+B).
\]
Also
$D_\Psi(\bar p^+,\bar p)\le C(A+B)v^2$ and
\[
\eta\langle F(p)+\tau\nabla\Psi(p),\bar p^+-\bar p\rangle
\le C\eta qv+C v^2(\mathcal U+A+B).
\]
Since $\tau\le q$ and $v$ is small, these inequalities yield
\begin{equation}
\label{eq:coupling}
\mathcal D^{\rm post}\le
(1-c_2v)\mathcal D+C\eta^2(q^2+(A+B)\tau^2)+\zeta
\end{equation}
for a universal $c_2>0$.

For each simplex block, the scalar inequality
\[
\sum_i\frac{(u_i-p_i)^2}{p_i}\le2D_\Psi(u,p)
\]
follows by putting $r=u_i/p_i$: use $p_i\le1/r$ when $r\ge1$
and $p_i\le1$ when $r\le1$. Together with the relative reference
movement it implies
\begin{equation}
\label{eq:coupling-noise}
\mathbb E[\zeta^2\mid\mathcal G_n]
\le C\eta^2q^2(\mathcal D+v^2),\qquad
|\zeta|\le Cc^{-1}\eta q(q/\tau+A+B).
\end{equation}
To see the variance bound without a dimension factor, put
$a_i=\bar p_i^+-p_i$. The preceding scalar inequality and
$|d_i|\le v/16$ give, summing over both blocks,
\[
\sum_i\frac{a_i^2}{p_i}
\le 2\sum_i\frac{(\bar p_i-p_i)^2}{p_i}
 +2\left(\frac v{16}\right)^2\sum_i\frac{\bar p_i^2}{p_i}
\le C(\mathcal D+v^2),
\]
where $\sum_i\bar p_i^2/p_i
=2+\sum_i(\bar p_i-p_i)^2/p_i$.
For either player, importance weighting bounds
$\mathbb E[\langle\widehat F,a\rangle^2\mid\mathcal G_n]$
by $q^2\sum_i a_i^2/p_i$ on that player's block.
Apply $(a+b)^2\le2a^2+2b^2$ to combine the two players, and
subtract the conditional mean to obtain~\eqref{eq:coupling-noise}.
For the range bound, the old floor and
Lemma~\ref{lem:movement} give
$p_i^{-1}\le Cc^{-1}\max\{q/\tau,A+B\}$.
The variance calculation thus allows the dependence caused by the
shared reward.

\paragraph{Predictable normalization.}
Set
\[
\nu_n=\frac{q_n^2}{(A+B)^2\tau_n^2},
\quad X_n=\nu_n\mathcal D_n,\qquad
J_n=\left[
\frac{1+2\beta_n}
{1-\beta_n/[2((1-\gamma)^{-1}+\sqrt{A+B})]}
\right]^2.
\]
Neither policy at this state changes between its visits, so
$\mathcal D_{n+1}=\mathcal D_n^{\rm post}$.
The boundary budgets give $\nu_{n+1}/\nu_n\le J_n\le1+C\beta_n$.
The actual ratio can depend on the current sample, but $J_n$ is
known before that sample. The right side of~\eqref{eq:coupling} is
nonnegative, since it bounds $\mathcal D_n^{\rm post}\ge0$.
We can therefore multiply the entire right side by $J_n\nu_n$.
The resulting noise $W_n=J_n\nu_n\zeta_n$ remains centered.

The required gain estimates follow by considering
$(A+B)\theta\le1$ and $(A+B)\theta\ge1$:
\begin{align*}
\frac{\eta q^4}{(A+B)^2\tau^3}&\le c^2,&
\frac{\eta q^2}{(A+B)\tau}&\le\frac{c^2}{A+B},\\
\frac{\eta q^3(q+(A+B)\tau)}{(A+B)^2\tau^3}&\le2c^2,&
\frac{\eta^3q^6}{(A+B)^4\tau^3}&\le c^6.
\end{align*}
The factor $(1+\alpha\theta)^{-1}$ in $\eta$ can only decrease these
expressions. Using~\eqref{eq:coupling}--\eqref{eq:coupling-noise} and
$\beta\le2cv$, we obtain
\[
X_{n+1}\le(1-c_3v_n)X_n+C v_n+W_n,
\quad
\mathbb E_n W_n^2\le C v_n(X_n+1),\quad |W_n|\le C.
\]
Here and below $\mathbb E_n$ is conditional on $\mathcal G_n$.
For small fixed $\lambda$, expanding the exponential of the bounded
centered variable $W_n$ gives
$\log\mathbb E_n e^{\lambda W_n}\le
C\lambda^2v_n(X_n+1)$, as required in
Lemma~\ref{lem:stochastic}.

\paragraph{Target error.}
Let $\mu_n=x_n^\top M_s(w_n)y_n$ and define
\[
e_n=\frac{z_n-\mu_n}{(A+B)\tau_n},\qquad
\xi_n=\frac{\kappa_n(Y_n-\mu_n)}{(A+B)\tau_n}.
\]
The variable $\xi_n$ is centered. Since the range of $Y_n$ is at
most $q_n$,
\[
\xi_n^2\le c\kappa_n,\qquad
\frac{\eta q^2}{(A+B)\tau\kappa}\le\frac{c}{A+B},
\qquad \frac{\beta}{\kappa}\le c^2.
\]
The exact target recursion and~\eqref{eq:mean-movement}, followed by
the predictable upper bound on $\tau_n/\tau_{n+1}$, give
\[
|e_{n+1}|\le(1+\kappa_n/8)
|(1-\kappa_n)e_n+\xi_n|+C\kappa_n.
\]
Squaring with Young's inequality of parameter $\kappa_n/8$ gives,
with $C_n=(1+\kappa_n/8)^3\le1+\kappa_n/2$,
\[
e_{n+1}^2\le(1-\kappa_n)e_n^2+C\kappa_n+
2C_n(1-\kappa_n)e_n\xi_n.
\]
The final term is centered and has conditional absolute value at
most $C\sqrt{\kappa_n e_n^2}$. The elementary exponential inequality for bounded variables
therefore bounds its conditional log moment
generating function by $C\lambda^2\kappa_n e_n^2$.
Thus $e_n^2$ also satisfies Lemma~\ref{lem:stochastic}.

\paragraph{All states and interaction times.}
Initially $\mathcal D_0=0$ and
$|z_0-\mu_0|\le(1-\gamma)/2$, so the initial energies are bounded
numerically. Apply Lemma~\ref{lem:stochastic} to the $2S$ processes.
Their dependence on each other is immaterial for the union bound.
For a time after a state's most recent visit $n$ and before its next
visit, the same bounds over partial intervals yield
\[
\frac{q_t^2\mathcal D_s(t)}{(A+B)^2\tau_t^2}
\le C(X_n+1),\qquad
\frac{|z_t(s)-\mu_s(t)|}{(A+B)\tau_t}
\le C(|e_n|+1).
\]
The first follows from the normalized policy recursion and its bounded
noise; the second follows from the absolute error recursion and
$|\xi_n|\le C$. The index $n$ refers to a past visit and satisfies
$n\le t$. Before a state's first visit the initial cover has not ended,
so its initial bound still applies.
Finally the barrier divergence dominates relative entropy blockwise,
because $\nabla^2\Psi(p)-\operatorname{diag}(1/p_i)$ is positive
semidefinite. The elementary relative entropy bound on total variation
gives $\|p-\bar p\|_1\le2\sqrt{\mathcal D}$ for the two blocks.
The matrix gap changes by at most $q\|p-\bar p\|_1$.
Lemma~\ref{lem:reference} now proves~\eqref{eq:local-tracking}.
All arguments over partial intervals include the quantities after the
updates at a closing round.
\end{proof}

\subsection{Joint control of the Bellman residual and temperature}
\label{app:residual}

The target error from Lemma~\ref{lem:local-learning} controls the bias in
the value correction. The following deterministic lemma shows why its
clipping threshold and the temperature decrease are compatible.

\begin{lemma}[Bellman consistency]
\label{lem:residual}
At a cover boundary, suppose
$\|z-\mathcal T w\|_\infty\le e_0\tau$ for some $e_0\ge0$.
Let $Q=8(e_0+1)$. The simultaneous value and temperature updates satisfy
\begin{equation}
\label{eq:joint-residual}
\left[\|\mathcal T w^+-w^+\|_\infty-Q\tau^+\right]_+
\le
\left[\|\mathcal T w-w\|_\infty-Q\tau\right]_+.
\end{equation}
Consequently, if the target bound holds at every boundary of a prefix
with the same $e_0$, then throughout that prefix
\[
\|\mathcal T w-w\|_\infty\le Q\tau,
\qquad
1-\frac{\tau^+}{\tau}\ge
\frac{\beta}{C((1-\gamma)^{-1}+\sqrt{A+B})(e_0+1)}.
\]
\end{lemma}

\begin{proof}
Write $u=\|\mathcal T w-w\|_\infty/\tau$ and
$a_0=e_0+1+2\beta$. A coordinate with residual magnitude at least
$(e_0+1)\tau$ has a discrepancy of the same sign and magnitude at
least $\tau$. Its value therefore moves by exactly $\beta\tau$
toward the old target. The new target changes by at most
$\gamma\beta\tau$, so this coordinate's residual magnitude drops
by at least $(1-\gamma)\beta\tau$. There is no sign crossing since
$\beta\le1/8$. Every remaining residual has new magnitude at most
$a_0\tau$. Hence
\[
\frac{\|\mathcal T w^+-w^+\|_\infty}{\tau}
\le\max\{u-(1-\gamma)\beta,a_0\}.
\]
Also $\|z-w\|_\infty\ge(u-e_0)_+\tau$, so the relative
temperature decrease is at most
\[
\frac{\beta}
{2((1-\gamma)^{-1}+\sqrt{A+B})\max\{1,u-e_0\}}.
\]
It remains to check
\begin{align*}
&\max\{u-(1-\gamma)\beta,a_0\}-Q\\
&\quad+
\frac{Q\beta}
{2((1-\gamma)^{-1}+\sqrt{A+B})\max\{1,u-e_0\}}
\le(u-Q)_+.
\end{align*}
For $u\le e_0+1$, the left side is negative by
$Q=8(e_0+1)$ and $\beta\le1/8$. For
$e_0+1\le u\le Q$, the branch from $a_0$ is again negative.
The other branch is convex as a function of $u$; it is negative at
$e_0+1$ and nonpositive at $Q$, using $Q\ge2e_0$ and
$((1-\gamma)^{-1}+\sqrt{A+B})^{-1}\le1-\gamma$.
For $u\ge Q$, subtract $u-Q$ and use the same inequality at $Q$.
Taking positive parts proves~\eqref{eq:joint-residual}.

Initially $\|\mathcal T w_0-w_0\|_\infty\le(1-\gamma)/2$ and
$\tau_0=512(1-\gamma)$, so induction gives the residual bound.
Finally $\|z-w\|_\infty\le(Q+e_0)\tau$, and substituting this
in~\eqref{eq:temperature-update} proves the temperature bound.
\end{proof}

\subsection{Counting covers and interactions}
\label{app:rate}

We use the hitting assumption in two places: to bound the length of a
cover and, through Lemma~\ref{lem:certificate}, to improve the observation
range. The first use follows the visiting argument of
\citet[Appendix~D.1]{cai2023uncoupled}; we include the details to cover
adaptive policies and our convention that a visited state must provide
a transition.

\begin{lemma}[Covers under adaptive play]
\label{lem:covers}
Under Assumption~\ref{ass:hitting}, the expected time to hit any target
state is at most $L$, conditionally on any history, under arbitrary
future adaptive play. For any $\delta\in(0,1)$, with probability at
least $1-\delta$, every cover starting at a time $j\ge0$ has length
at most
\[
C(L+1)\log\frac{CS(j+2)^2}{\delta}.
\]
\end{lemma}

\begin{proof}
Fix a target and regard a pair of actions as one action in a finite
Markov decision process. For a discount $\zeta\in(0,1)$, maximize
the expected discounted number of steps before hitting the target.
The finite discounted Bellman equation has an optimal stationary
deterministic policy. Its joint action at each state is a stationary
product policy, so its undiscounted expected hitting time is at most
$L$. The discounted maximum is therefore at most $L$.
Every adaptive policy has discounted cost at most this maximum.
Letting $\zeta\uparrow1$ by monotone convergence proves the claim
for arbitrary policies, including conditionally on a given history.

Markov's inequality now gives conditional probability at most $1/2$
of failing to hit a specified state in $\lceil2(L+1)\rceil$ steps.
Iteration in successive blocks gives a geometric tail. After reaching
a state, include its next transition, costing one extra step.
A union bound over $S$ targets bounds the probability that a cover
starting at $j$ exceeds $C(L+1)\log(CS(j+2)^2/\delta)$ by
$\delta/[C(j+2)^2]$. Union over all integer starting times proves
the assertion. In particular all local visit sequences are infinite
almost surely.
\end{proof}

\begin{lemma}[Temperature count]
\label{lem:temperature-count}
Fix an analytical prefix and a number $H\ge1$. Suppose throughout
that prefix the conclusions of Lemma~\ref{lem:local-learning} hold
with $\sqrt{\Lambda_t}$ replaced by $H$. Put
$\rho=\min\{1,(L+1)(1-\gamma)\}$. For every
\[
0<u\le\frac{1-\gamma}{C(A+B)H},
\]
the number of completed covers in the prefix whose starting
temperature exceeds $u$ is at most
\begin{equation}
\label{eq:cover-count}
\frac{C((1-\gamma)^{-1}+\sqrt{A+B})H}{A+B}
\rho^4\log\frac e\rho\;u^{-4}.
\end{equation}
Moreover, at every time in the prefix,
\begin{equation}
\label{eq:gap-temperature}
\operatorname{Gap}(x_t,y_t)
\le\frac{C(A+B)H\tau_t}{(1-\gamma)^2}.
\end{equation}
\end{lemma}

\begin{proof}
The played payoff and the local game value both lie between
$\min_v x^\top Mv$ and $\max_u u^\top My$.
Thus the assumed bounds give
$\|z-\mathcal T w\|_\infty\le C(A+B)H\tau$.
Apply Lemma~\ref{lem:residual} with the fixed value
$e_0=C(A+B)H$. Lemma~\ref{lem:certificate} then proves
\eqref{eq:gap-temperature}, and the temperature decreases at each
boundary by a relative amount at least
\begin{equation}
\label{eq:clock-progress}
\frac{c_4\beta}
{((1-\gamma)^{-1}+\sqrt{A+B})(A+B)H}
\end{equation}
for a positive constant $c_4$.
The span bound~\eqref{eq:span-residual} gives, when $\rho<1$,
\[
q\le\rho+F\tau,
\qquad F=C(A+B)(L+1)H.
\]
Always $q\le513$; when $\rho=1$ this alone suffices below.

First consider covers with $\tau>1/(A+B)$. Since $q\le513$,
$\varepsilon\ge1/513$ and $\beta\ge c_5\tau^2$.
By~\eqref{eq:clock-progress}, each such cover increases $\tau^{-2}$
by at least
$c_6/[((1-\gamma)^{-1}+\sqrt{A+B})(A+B)H]$.
There are at most
\[
C((1-\gamma)^{-1}+\sqrt{A+B})(A+B)^3H
\]
such covers. The possible last step crossing the threshold changes
$\tau$ only by a fixed factor and is included in this bound.

For $\tau\le1/(A+B)$ and $\rho<1$, set
$Q_*(\tau)=\min\{513,\rho+F\tau\}$. We may choose $F\ge A+B$,
so $Q_*(\tau)\ge(A+B)\tau$. Since
$\beta=c^3\theta^2\min\{1,(A+B)^2\theta^2\}$ decreases with $q$,
\[
\beta\ge
\frac{c^3(A+B)^2\tau^4}{Q_*(\tau)^4}.
\]
Hence the number of covers in a band $\vartheta/2<\tau\le\vartheta$
is at most
\[
\frac{C((1-\gamma)^{-1}+\sqrt{A+B})H}{A+B}
\min\left\{513^4\vartheta^{-4},
\left(\frac\rho\vartheta+F\right)^4\right\}.
\]
Indeed,~\eqref{eq:clock-progress} supplies a lower bound on each
relative decrease throughout the band; a total log decrease of
$\log2$ leaves the band. An additional single crossing step is
absorbed in the displayed bound.

Sum these bands down to $u$. For $\vartheta\ge513/F$, their
contributions form a geometric sum bounded by $CF^4$ times the
common prefactor. For $\rho/F\le\vartheta\le513/F$, each
contribution is at most $CF^4$ times that prefactor, and there are
at most $C\log(e/\rho)$ bands. Below $\rho/F$, the geometric
sum is at most $C\rho^4u^{-4}$ times the prefactor.
Thus the count is bounded by
\[
\frac{C((1-\gamma)^{-1}+\sqrt{A+B})H}{A+B}
\left(\rho^4u^{-4}+F^4\log\frac e\rho+(A+B)^4\right).
\]
The restriction on $u$ and $\rho=(L+1)(1-\gamma)$ imply
$Fu/\rho\le C$ and $(A+B)u/\rho\le C$.
They absorb both additional terms into~\eqref{eq:cover-count}.
When $\rho=1$, use $q\le513$ directly in every band. The resulting
geometric sum is $Cu^{-4}$, and the same bound for warm covers is
absorbed by the restriction on $u$.
\end{proof}

\begin{proof}[Proof of Theorem~\ref{thm:last-iterate}]
Intersect the events of Lemmas~\ref{lem:local-learning}
and~\ref{lem:covers}, each with failure probability $\delta/2$.
Their logarithms are bounded by numerical multiples of $\Lambda_t$.
Fix any integer $T\ge1$ on this event, and put $H=C\sqrt{\Lambda_T}$.
Every cover beginning by time $T$ has length at most
$B_T=C(L+1)\Lambda_T$.

Let $g\in(0,(1-\gamma)^{-1}]$ be a target gap, and take
$u=g(1-\gamma)^2/[C(A+B)H]$, with the constant large enough for
Lemma~\ref{lem:temperature-count}. If $\tau_T>u$, all completed
covers up to $T$ start above $u$ and obey~\eqref{eq:cover-count}.
There is at most one additional unfinished cover. Multiplying that
count by $B_T$ and absorbing the extra cover shows
\[
T\le
\frac{C(A+B)^3}{(1-\gamma)^8}
\left(\frac1{1-\gamma}+\sqrt{A+B}\right)
(L+1)\rho^4\log\frac e\rho\;
\Lambda_T^{7/2}g^{-4}.
\]
Therefore whenever $T$ exceeds this expression,
$\tau_T\le u$ and~\eqref{eq:gap-temperature} gives
$\operatorname{Gap}(x_T,y_T)\le g$.
Inverting the displayed implication, with a larger numerical constant,
gives the sharper bound
\begin{equation}
\label{eq:sharp-main-rate}
\operatorname{Gap}(x_T,y_T)
\le
\frac{C(A+B)^{3/4}}{(1-\gamma)^2}
\left[
\left(\frac{1}{1-\gamma}+\sqrt{A+B}\right)
(L+1)\rho^4\log\frac{e}{\rho}
\right]^{1/4}
\frac{\Lambda_T^{7/8}}{(T+1)^{1/4}}.
\end{equation}
If the right-hand side exceeds $(1-\gamma)^{-1}$, the inequality
follows from the elementary bound that all unnormalized values lie in
$[0,(1-\gamma)^{-1}]$. This also covers $T=0$, since
$\rho\ge1-\gamma$ makes the right-hand side at least
$(1-\gamma)^{-1}$ after increasing $C$. Replacing $T$ by $T+1$
only changes a numerical constant for $T\ge1$.

To simplify the parameter dependence, use $\rho\le1$ and
\[
\frac{1}{1-\gamma}+\sqrt{A+B}
\le \frac{2(A+B)}{1-\gamma}.
\]
Substituting into~\eqref{eq:sharp-main-rate} gives
\[
\operatorname{Gap}(x_T,y_T)
\le
\frac{C(A+B)(L+1)^{1/4}}{(1-\gamma)^{9/4}(T+1)^{1/4}}
\Lambda_T^{7/8}\left(\log\frac{e}{\rho}\right)^{1/4}.
\]
Since $\log(e/\rho)\le\log(e/(1-\gamma))$, absorbing the logarithmic
factors proves~\eqref{eq:main-rate}. The argument holds for every $T$
on the same event. The analytical prefix and confidence level never
enter the updates.
\end{proof}

\section{Proof of Corollary~\ref{cor:rationality}}
\label{app:rationality}

We prove the result for the first player. In this appendix constants
may depend on the fixed game parameters. We work on the event of probability one
on which the opponent's policies converge and all the visiting
bounds below hold. The limiting policy can be random. Whenever it enters
an expression, that expression is used pathwise; conditional expectations
are always taken around the current, predictable opponent policy.

The proof has three steps. First the gains retain enough cumulative
weight for stochastic approximation. Second reciprocal probabilities
make the learner asymptotically greedy against the limiting opponent.
Finally the target and value estimates approach its best response value.

\subsection{Gains and vanishing filtered noise}

\begin{lemma}[Gain bounds against an arbitrary opponent]
\label{lem:rational-gains}
Let $k$ count completed covers of one learner. Its temperature tends to
zero and satisfies, for fixed positive constants,
\[
c_7(k+1)^{-1/4}\le\tau_k\le C_7(k+1)^{-1/5}.
\]
Almost surely $\tau_t\le C(\log(t+2)/(t+1))^{1/5}$ eventually.
At the $n$th visit to any fixed state,
\[
\sum_n v_n=\sum_n\kappa_n=\infty,\qquad
\eta_n\le C\left(\frac{\log(n+2)}{n+1}\right)^{3/5},\qquad
\kappa_n\le C\left(\frac{\log(n+2)}{n+1}\right)^{2/5}
\]
eventually almost surely. The boundary gains
\[
b_k=\frac{\beta_k\tau_k}
{\max\{\tau_k,\|z_k-w_k\|_\infty\}}
\]
satisfy $\sum_k b_k=\infty$.
\end{lemma}

\begin{proof}
The value and scale bounds and the learner's policy movement bound in
Lemma~\ref{lem:movement} hold against any opponent. If the decreasing
temperature had a positive limit, then $q\in[1-\gamma,513]$ would
keep $\beta$ bounded below. Since $\|z-w\|_\infty\le1$, the
temperature would decrease by a fixed relative amount at every cover,
a contradiction. All covers end almost surely by
Lemma~\ref{lem:covers}.

For a fixed game and small $\tau$, the gain definitions give
\[
\eta\asymp\tau^3,\qquad
v\asymp\beta\asymp\tau^4,\qquad
\kappa\asymp\tau^2.
\]
The relative temperature decrease is between $c\tau^5$ and
$C\tau^4$, using $\|z-w\|_\infty\le1$.
Thus $\tau_k^{-5}$ increases by at least a positive constant,
whereas $\tau_k^{-4}$ increases by at most a constant.
This proves the cover bounds, with the finite initial period absorbed.
The geometric tails in Lemma~\ref{lem:covers} and a summable union
bound imply that, almost surely, a cover starting at time $t$ has
length at most $C\log(t+2)$ eventually. Hence at least
$ct/\log(t+2)$ covers have completed by time $t$, proving the
upper bound in interaction time.

Before a state's $n$th visit, at most $n$ covers can have ended.
Consequently $\tau_n\ge c(n+1)^{-1/4}$,
$v_n\ge c/(n+1)$, and $\kappa_n\ge c/(n+1)^{1/2}$.
Conversely its visit time is at least $n$, which gives the displayed
upper bounds. Finally
\[
\frac{\tau_{k+1}}{\tau_k}
=1-\frac{b_k}{2((1-\gamma)^{-1}+\sqrt{A+B})}.
\]
Since $\tau_k\to0$ and $b_k$ is uniformly small, taking logarithms
shows $\sum_k b_k=\infty$.
\end{proof}

\begin{lemma}[Vanishing noise in stable scalar recursions]
\label{lem:vanishing-noise}
Let predictable gains $a_n\in(0,1)$ satisfy $a_n\to0$ and
$\sum_n a_n=\infty$ almost surely. Let $\xi_n$ be martingale
differences relative to $(\mathcal G_n)$, and write
$\mathbb E_n=\mathbb E[\,\cdot\mid\mathcal G_n]$. Suppose
\[
|\xi_n|\le C\sigma_n,\qquad
\mathbb E_n\xi_n^2\le C\sigma_na_n,
\]
where $\sigma_n$ is predictable and
$\sigma_n\le C(n+1)^{-r}\log^C(n+2)$ eventually almost surely for
some $r>0$. Let $d_n/a_n\to0$ almost surely and $\chi_n\in[0,1]$.
Then each of the following recursions, with a finite initial value,
converges to zero almost surely:
\[
X_{n+1}=\chi_n\bigl((1-a_n)X_n+d_n+\xi_n\bigr),
\qquad
0\le Z_{n+1}\le
\chi_n\bigl[(1-a_n)Z_n+d_n+\xi_n\bigr]_+.
\]
The factors $\chi_n$ and forcing $d_n$ need not be predictable.
\end{lemma}

\begin{proof}
We give the argument for weighted noise, including history dependent
factors. Partition sufficiently late indices into blocks on which
the cumulative gain first reaches one. Each block has gain at most
$1+\sup a_n\le2$. For a block starting at $m$, the conditional
variance of its unweighted martingale sum is at most
$C m^{-r}\log^C(m+2)$ and each increment has the same upper bound
in absolute value, with constants enlarged if necessary.
The exponential martingale inequality for bounded increments gives,
for any fixed $\epsilon>0$,
\[
\mathbb P\left(\max_{j\text{ in the block}}
\left|\sum_{n=m}^{j}\xi_n\right|>\epsilon\right)
\le 2\exp\left\{
-\frac{c\epsilon^2 m^r}{(1+\epsilon)\log^C(m+2)}\right\}.
\]
This inequality follows from
$\mathbb E_n e^{\lambda\xi_n}\le
\exp(C\lambda^2\mathbb E_n\xi_n^2)$ for
$\lambda\max|\xi_n|\le1/2$, followed by the maximal inequality
and optimization in $\lambda$.
Its right side is summable over all integer starting points $m$.
If the eventual bounds have random constants, first stop upon their
failure for deterministic constants and a deterministic starting
index, then take the countable union of these localizations.
Borel--Cantelli implies that the maximum partial sum in every late
gain block tends to zero.

Within a block ending at $j$, the coefficient of the noise at index
$i\le j$ in the solution of the recursion is
\[
w_{i,j}=\chi_i\prod_{k=i+1}^{j}\chi_k(1-a_k).
\]
It belongs to $[0,1]$, and
$w_{i,j}=\chi_i(1-a_{i+1})w_{i+1,j}\le w_{i+1,j}$.
Thus these coefficients have total variation at most one as a
function of the noise index, even when the factors depend on later
observations.
Summation by parts bounds each weighted sum by twice the maximum
unweighted partial sum. Across complete blocks the old value
contracts by at least $e^{-1}$. The forcing $o(a_n)$ contributes
$o(1)$ per block. These facts imply convergence to zero, first at
block endpoints and then inside every block. For a recursion reflected at zero,
unrolling it takes the maximum over
suffixes of the same weighted sums; a suffix partial sum is bounded
by twice the corresponding prefix maximum. The same argument applies.
\end{proof}

\subsection{Reciprocal probabilities and asymptotic greediness}

\begin{lemma}[Greediness against a limiting opponent]
\label{lem:greediness}
Suppose $y_t\to y_\infty$ statewise almost surely. For every state,
the first player satisfies
\[
\max_{a\in\mathcal A}[M_s(w_t)y_\infty(\cdot\mid s)]_a
-x_t(\cdot\mid s)^\top M_s(w_t)y_\infty(\cdot\mid s)\longrightarrow0
\qquad\text{almost surely}.
\]
\end{lemma}

\begin{proof}
Fix a state, use its local visit index $n$, and write $p_n$ for the
learner's distribution. Put
\[
u_{n,i}=\frac{\tau_n}{p_{n,i}},\qquad
\rho_n=\frac{\tau_{n+1}}{\tau_n}.
\]
The old floor implies $u_{n,i}\le C$ for this fixed game.
Moreover $1-Cv_n\le\rho_n\le1$ by the boundary budgets.
Let $F_{n,i}$ be the conditional mean centered loss coordinate
against the current opponent, and let $\widehat F_{n,i}$ be its
importance estimate. The optimality conditions give
\begin{equation}
\label{eq:reciprocal-step}
u_{n+1,i}=\min\left\{
\frac{\rho_n\tau_n}{h_n},\;
\rho_n\bigl((1-v_n)u_{n,i}+v_n\widehat F_{n,i}
+\tau_n\lambda_n\bigr)\right\}.
\end{equation}
Here $\lambda_n$ is the simplex multiplier. It is common to all
coordinates, which is why differences are useful.

We first show that the floor is eventually inactive. Taking the
positive part of a coordinate difference in~\eqref{eq:reciprocal-step},
using the common cap and $\rho_n\le1$, gives
\[
(u_{n+1,i}-u_{n+1,j})_+
\le\left[(1-v_n)(u_{n,i}-u_{n,j})
+v_n(F_{n,i}-F_{n,j})+\xi_{n,ij}\right]_+,
\]
where
$\xi_{n,ij}=v_n[(\widehat F_{n,i}-\widehat F_{n,j})
-(F_{n,i}-F_{n,j})]$ is centered.
The floor and bounded observations imply
\[
|\xi_{n,ij}|\le C\eta_n,\qquad
\mathbb E_n\xi_{n,ij}^2\le C\eta_nv_n.
\]
The conditional loss range is at most $q_n$, and
$|q_{n+1}-q_n|\le C\beta_n\tau_n=o(v_n)$.
It follows that $W_n=(u_{n,i}-u_{n,j}-q_n)_+$ satisfies
\[
W_{n+1}\le[(1-v_n)W_n+\xi_{n,ij}+o(v_n)]_+.
\]
Lemmas~\ref{lem:rational-gains} and~\ref{lem:vanishing-noise}
give $W_n\to0$. Since $\min_i u_{n,i}\le d\tau_n\to0$,
we have $\max_i u_{n,i}\le q_n+o(1)$.
An active floor at visit $n$ would instead give, at the next visit,
\[
u_{n+1,i}=\frac{\rho_n\tau_n}{h_n}
=c^{-1}\rho_n\max\{q_n,(A+B)\tau_n\}.
\]
This contradicts the preceding upper bound for all large $n$,
because $q_n\ge1-\gamma>0$, $\rho_n\to1$, and $c$ is small.
Thus~\eqref{eq:reciprocal-step} is eventually uncapped.

Let $f_{n,i}$ be the centered loss coordinate obtained by replacing
$y_n$ with $y_\infty$, at the same $w_n$.
The convergence of the opponent implies
$F_{n,i}-f_{n,i}\to0$ uniformly in $i$.
Although $f_n$ need not be predictable, this is a pathwise error;
the martingale is still centered at $F_n$.
The fixed limiting opponent and the boundary budgets imply
$|f_{n+1,i}-f_{n,i}|=\mathcal O(\beta_n\tau_n)=o(v_n)$.
Fix a coordinate $i_0$ and, after the floor has become inactive,
define an analytical scalar $\omega_n$ by
\[
\omega_{n+1}=\rho_n[(1-v_n)\omega_n+v_n],\qquad
\omega_{n_0}=1.
\]
It lies in $[\omega_*,1]$ for a fixed $\omega_*>0$:
the lower inequality $\rho_n\ge1-Cv_n$ makes any sufficiently
small positive $\omega_*$ invariant for all late $n$.
Subtracting the uncapped coordinate recursions shows that
\[
E_{n,i}=u_{n,i}-u_{n,i_0}
-\omega_n(f_{n,i}-f_{n,i_0})
\]
obeys
$E_{n+1,i}=\rho_n[(1-v_n)E_{n,i}+\xi_{n,ii_0}+o(v_n)]$.
The final $o(v_n)$ also includes the change of $f_n$; $\rho_n$
is bounded away from zero. Lemma~\ref{lem:vanishing-noise} gives
$E_{n,i}\to0$ for all $i$.

Choose a coordinate $j$ minimizing $f_n$. The preceding identity
for differences implies
\[
\omega_n(\langle p_n,f_n\rangle-f_{n,j})
\le \sum_i p_{n,i}u_{n,i}+2\max_i|E_{n,i}|
=d\tau_n+2\max_i|E_{n,i}|\longrightarrow0.
\]
Since $\omega_n\ge\omega_*$, the loss regret tends to zero.
These losses are negatives of rewards up to a common shift, proving
the assertion at visits. Between visits the policy is fixed and
there is at most one value correction of size $\beta\tau\to0$,
which extends the assertion to every interaction time.
\end{proof}

\subsection{Convergence of the target and value estimates}

\begin{proof}[Proof of Corollary~\ref{cor:rationality}]
Let $\mathcal T_\infty$ be the first player's best response Bellman
operator with $y_\infty$ fixed, and let its normalized fixed point
be $w_\infty^*$. At a fixed state's visits write
\[
\mu_n=x_n^\top M_s(w_n)y_n,\qquad
\widetilde\mu_n=x_n^\top M_s(w_n)y_\infty.
\]
Then $\mu_n-\widetilde\mu_n\to0$.
Lemma~\ref{lem:movement} applied to the learner's own policy, while
keeping the limiting opponent fixed, gives
\[
|\widetilde\mu_{n+1}-\widetilde\mu_n|
\le C\eta_n q_n^2+C\beta_n\tau_n=o(\kappa_n).
\]
Thus the averaging recursion for $z_n-\widetilde\mu_n$ has
contraction $1-\kappa_n$, forcing $o(\kappa_n)$, and martingale
increment $\kappa_n(Y_n-\mu_n)$.
Its absolute value is at most $\kappa_n$ and its conditional
variance at most $\kappa_n^2$.
Lemmas~\ref{lem:rational-gains} and~\ref{lem:vanishing-noise}
therefore show $z_n-\widetilde\mu_n\to0$.
As in the preceding lemma, the bounds over partial intervals extend this
to all times. Lemma~\ref{lem:greediness} now yields
\[
e_k:=\|z_k-\mathcal T_\infty w_k\|_\infty\longrightarrow0
\]
at cover boundaries, using the estimates after the closing sample.

We verify convergence of the clipped value correction without imposing
a rate on $e_k$. Let $Q_k=\|w_k-w_\infty^*\|_\infty$.
The scalar map $u+\beta\operatorname{clip}(v-u,-\tau,\tau)$ is
nondecreasing in both arguments when $\beta\le1$.
Contraction gives
$z_k\le w_\infty^*+\gamma Q_k\mathbf1+e_k\mathbf1$, with the
symmetric lower bound. Consequently, if
$(1-\gamma)Q_k\ge e_k$, then
\[
Q_{k+1}\le Q_k-
\beta_k\min\{\tau_k,(1-\gamma)Q_k-e_k\}.
\]
For a fixed $\epsilon>0$, once $e_k\le(1-\gamma)\epsilon/2$
and $\tau_k\le(1-\gamma)\epsilon/2$, the box
$\|w-w_\infty^*\|_\infty\le\epsilon$ is invariant.
While $Q_k>\epsilon$, the displayed inequality decreases $Q_k$
by $\beta_k\tau_k$. Furthermore
\[
\|z_k-w_k\|_\infty
\ge(1-\gamma)Q_k-e_k
\ge(1-\gamma)\epsilon/2.
\]
Therefore $\beta_k\tau_k\ge
((1-\gamma)\epsilon/2)b_k$ during this period.
Since $\sum_k b_k=\infty$, the process must enter that box.
It follows that $w_k\to w_\infty^*$, and hence $w_t\to w_\infty^*$.

Finally let $g_t$ be the maximum local greediness error in
Lemma~\ref{lem:greediness}. The normalized Bellman evaluation operator
for $(x_t,y_\infty)$ satisfies
\[
\|\mathcal T_\infty w_\infty^*
-\mathcal T_{x_t,y_\infty}w_\infty^*\|_\infty
\le g_t+2\gamma\|w_t-w_\infty^*\|_\infty\longrightarrow0.
\]
Contraction of the evaluation operator bounds the distance between its
fixed point and $w_\infty^*$ by this quantity divided by $1-\gamma$.
Dividing once more by $1-\gamma$ to recover unnormalized values proves
the stated best response convergence.

For the second player, interchange the players and replace the reward
by $1-r_t$ and the estimates by $1-z_t$ and $1-w_t$.
The transformed TD target is $1-Y_t$, and the first player's loss formula
in this game equals the original second player's loss.
The span, discrepancy norm, temperature, and all learning rates are
unchanged, and the averaging and clipped value updates commute with
this transformation. The argument for the first player therefore applies
to the transformed game, whose rewards also lie in $[0,1]$.
Since the transformed value equals $(1-\gamma)^{-1}-V^{x,y}$,
its best response conclusion is precisely the required statement for
the second player.
\end{proof}